\documentclass{article}

    \PassOptionsToPackage{square, numbers}{natbib}
    \usepackage[preprint]{neurips_2020}

\usepackage[utf8]{inputenc} % allow utf-8 input
\usepackage[T1]{fontenc}    % use 8-bit T1 fonts
\usepackage{hyperref}       % hyperlinks
\usepackage{url}            % simple URL typesetting
\usepackage{booktabs}       % professional-quality tables
\usepackage{nicefrac}       % compact symbols for 1/2, etc.
\usepackage{microtype}      % microtypography

\usepackage{amsmath}
\usepackage{amssymb}
\usepackage{amsthm}
\usepackage{mathtools}
\mathtoolsset{showonlyrefs}
\allowdisplaybreaks
\newcommand{\E}{\mathbb{E}}
\newcommand{\pr}{\mathbb{P}}
\newcommand{\R}{\mathbb{R}}

\DeclareMathOperator*{\argmin}{arg\,min}
\newtheorem{proposition}{Proposition}

\usepackage{graphicx}
\usepackage[ruled,vlined]{algorithm2e}

\title{Open-Set Recognition with Gaussian Mixture Variational Autoencoders}

\author{%
  Alexander Cao\\
  Department of Industrial Engineering and Management Sciences\\
  Northwestern University\\
  Evanston, IL 60208\\
  \texttt{a-cao@u.northwestern.edu}\\
  \And
  Yuan Luo\\
  Department of Preventive Medicine\\
  Northwestern University\\
  Chicago, IL 60611\\
  \texttt{yuan.luo@northwestern.edu}\\
  \And
  Diego Klabjan\\
  Department of Industrial Engineering and Management Sciences\\
  Northwestern University\\
  Evanston, IL 60208\\
  \texttt{d-klabjan@northwestern.edu} \\
}

\begin{document}

\maketitle

\begin{abstract}
In inference, open-set classification is to either classify a sample into a known class from training or reject it as an unknown class. Existing deep open-set classifiers train explicit closed-set classifiers, in some cases disjointly utilizing reconstruction, which we find dilutes the latent representation's ability to distinguish unknown classes. In contrast, we train our model to cooperatively learn reconstruction and perform class-based clustering in the latent space. With this, our Gaussian mixture variational autoencoder (GMVAE) achieves more accurate and robust open-set classification results, with an average F1 improvement of 29.5\%, through extensive experiments aided by analytical results.
\end{abstract}

\section{Introduction}
Until recently, nearly all classification algorithms have been designed for closed-set evaluation. This means that all testing classes are seen in training. However, real-world applications necessitate open-set evaluation where unknown classes, not seen in training, appear during testing. For instance, computer vision systems in self-driving cars must classify and navigate around many different objects. Given the countless number of such possible objects, it is infeasible for all classes to be seen in training \cite{limits}. Open-set recognition addresses this generalization of the classification task. 

While there are several facets of open-set learning, in this paper we focus on training from $C$ known classes for $(C+1)$-class classification. This $(C+1)$-th class catches all unknown test samples not belonging to any of the known classes. The training data has no unseen classes from class $C+1$. To this end, we present a novel supervised, Gaussian mixture variational autoencoder (GMVAE). The bottleneck latent layer simultaneously learns reconstruction and performs class-based clustering. This allows the latent representation to capture complementary structure and classifier information. Furthermore, the latent layer has the explicit capability to form multiple subclusters per class. This challenges the implicit assumption made by many classification methods that a class's embedding is a convex set and thus is best represented by a single centroid \cite{TowardsOpenSetDeepNetworks, inter-intra, crosr}. This provides further flexibility in capturing complementary structure and classifier information

Our contributions are as follows. In \S \ref{section:model}, we derive GMVAE to learn the embedding and amend its objective function to make open-set recognition more amenable. We also present a new and simple open-set classification algorithm that utilizes an ``uncertainty'' threshold on the learned embedding. Following in \S \ref{section:analytical}, we present analytical results regarding the number of subclusters and the resulting heuristic procedure for identifying the appropriate number of subclusters in each class. Finally in \S \ref{section:experiments}, we conduct open-set classification experiments on three standard datasets. Our findings from experiments are two-fold. First, GMVAE outperforms the state-of-the-art deep open-set classifier both in terms of accuracy and robustness to an increasing number of unknown classes. Second, the use of extreme value theory (EVT) to infer class-belongingness \cite{TowardsOpenSetDeepNetworks, crosr} may be ill-suited as we find that ours and another simple algorithm consistently beat it. 

\section{Related work}
While closed-set classification has been well-studied, open-set recognition has been somewhat dormant with the majority of its work appearing in the last decade. Outlier or novelty detection is a precursor but is not generally concerned with distinguishing between the known classes \cite{survey, anomaly}. Earlier works that study $(C+1)$-class classification utilize, for example, SVM scores \cite{TowardsOpenSetRecognition, ProbabilityofInclusion} or sparse representation \cite{sparse} to fit EVT-based densities to predict classes. The use of deep networks in open-set recognition appears even more recently in studies such as \cite{TowardsOpenSetDeepNetworks, crosr}. Both use similar procedures of fitting EVT-based densities to the distances between a class's embedding and its centroid to approximate probability of class inclusion. 

In this paper, our experimental results are benchmarked against the Classification-Reconstruction learning for Open-Set Recognition (CROSR) method \cite{crosr}. We chose this particular benchmark as it achieves state-of-the-art open-set classification accuracies and it relies on a similar framework of dual reconstruction-classification learning.

We next summarize CROSR. The latent representation is a concatenation $[y,z]$ where $y$ is the activation vector of a closed-set, softmax classifier and $z$ is the reconstructive latent representation. To learn an effective $y$ and $z$ concurrently, \cite{crosr} introduced Deep Hierarchical Reconstruction Nets (DHRNets). Conceptually, the DHRNet architecture is a deep classifier $f$ with autoencoder networks $h_l, \widetilde{h}_l$ appended at the internal layers $x_l$. Thus, bottleneck representations can be extracted from multi-stage features of the classifier. The autoencoders' reconstructions then form a reverse network to reconstruct the original input. Mathematically, the main-body network $f(x) = (y,z)$ is comprised of
\begin{align}
x_{l+1} &= f_l (x_l) \quad \text{$l$-th layer of the DHRNet classifer},~ z_l = h_l (x_l) \quad \text{encoder network for $l$-th layer} \\
\widetilde{x}_l &= g_l(\widetilde{x}_{l+1} + \widetilde{h}_l (z_l)) \quad \text{decoder network $\widetilde{h}_l$ and reconstruction network $g_l$ for $l$-th layer}
\end{align}
where networks are a series of convolutions and up or down-sampling layers. For training, \cite{crosr} minimizes the sum of the cross-entropy classification error and the $L_2$ reconstruction errors.

With latent representation $[y,z]$ in hand, CROSR applies EVT by fitting a Weibull distribution to the hypersphere defined by $d(x, C_i) = | [y,z] - \mu_i |_2$ where $\mu_i$ is the respective mean within class $C_i$. A proxy for probability of class inclusion is then given by
$
\pr (x \in C_i) =  1 - \text{WeibullCDF}(d(x, C_i); \rho_i) =  \exp \left\{ - \left(  \frac{d(x,C_i)}{\eta_i}  \right)^{m_i} \right\}
$
and thresholding is then used to classify a sample as ``unknown.'' Here $m_i$ and $\eta_i$ are parameters of the distribution fitted from class $C_i$'s training data.

In contrast to DHRNets, Gaussian mixture variational autoencoders from \cite{gmvae} are deep generative models which estimate the density of training data under assumptions on its latent prior. This could lead to more complex latent structures than in classification-based models. However, inference in this unsupervised setting is challenging, especially with open-set recognition. We address this by extending this deep generative model to supervised learning including capturing subclusters within classes.

\section{Gaussian mixture variational autoencoders} \label{section:model}
In this section we present our complete, novel procedure for open-set recognition. It follows the same two phases as previous works: first, learn a latent representation to (sub)cluster known classes, and second, apply an open-set classification algorithm on that embedding. Our GMVAE model is an extension of the Gaussian mixture variational autoencoder presented in \cite{gmvae} and explained next. 

Variational autoencoders (VAEs) assume data is generated from a uni-modal Gaussian prior. In \cite{gmvae}, the authors instead choose a mixture of Gaussians as an intuitive extension. In order to maintain standard backpropagation via the reparametrisation trick, the standard VAE architecture was altered. The generative model, factorizing as $p_{\beta, \theta}(x,z,w,v) = p(w) p(v) p_\beta  (z | w,v) p_\theta (x|z)$, generates a sample $x$ from the latent variables $z$, $w$, and $v$ with the following process
\begin{align}
w & \sim \mathcal{N} (0, I), \quad v \sim \text{Mult} (\pi) \\
(z|w,v) &\sim \prod_{k=1}^K \mathcal{N} \left(\mu_{k} (w; \beta), \text{diag}\left( \sigma^2_{k} (w; \beta)\right) \right)^{v_k} \\
(x | z) &\sim \mathcal{N} \left( \mu(z; \theta), \text{diag} \left( \sigma^2(z; \theta) \right) \right) \quad \text{or} \quad \mathcal{B} \left( \mu(z; \theta)\right)
\end{align}
where $K$ is the user-defined number of mixture components and $\mu_{k} (\cdot; \beta)$, $\sigma^2_{k} (\cdot; \beta)$, $\mu(\cdot; \theta)$, and $\sigma^2(\cdot; \theta)$ are neural networks parametrized by $\beta$ and $\theta$, respectively. The recognition model is then factorized as $q(z,w,v|x) = q_{\phi_z} (z|x) q_{\phi_w} (w|x) p_\beta (v|z,w)$ where $\phi_z$  and $\phi_w$ parametrize neural networks that output means and diagonal covariances of the Gaussian posterior variational distributions. Using Bayes' rule, the $v$-posterior term $p_\beta (v|z,w)$ can be written in terms of factors of the generative model. To train, the log-evidence lower bound (ELBO) $\E_{q(z,w,v|x)} \left[p_{\beta, \theta}(x,z,w,v)  / q(z,w,v|x) \right]$ is maximized. In \S \ref{section:s-gmvae} and \ref{section:no-v}, we present the derivation and differences of our GMVAE. Finally we introduce our new open-set classification algorithm that utilizes an ``uncertainty'' threshold in \S \ref{section:algorithms}.

\subsection{Gaussian mixture variational autoencoders with multiple subclusters per class} \label{section:s-gmvae}
Our GMVAE model nontrivially extends the unsupervised learning framework of \cite{gmvae} to essentially a Gaussian mixture prior for each class. For notation, there are $C$ known classes with each class composed of $K_c$ subclusters where $c = 1, 2, ..., C$. The samples $x \in \R^d$ and labels $y\in\R^C$ as one-hot vectors comprise the labeled, known data set $(x,y) \in \mathcal{X}$. The GMVAE's generative process $p_{\beta,\theta}(x,v,w,z|y) = p_\theta (x|z) p_\beta(z|w,y,v)p(w) p(v|y)$ is conditioned on class and given by
\begin{align}
w &\sim \mathcal{N} (0, I), \quad (v | y ) \in \R^{K_c} \sim \text{Mult} (\pi(y)) \\
(z|w,y,v) &\sim \prod_{c=1}^C \prod_{k=1}^{K_c} \mathcal{N} \left(\mu_{ck} (w; \beta), \text{diag}\left( \sigma^2_{ck} (w; \beta)\right) \right)^{y_c\cdot v_k} \\
(x | z) &\sim \mathcal{B} \left( \mu(z; \theta)\right). 
\end{align}
It is common to take $\pi(y)$ to simply be uniform for each class. The recognition model is factorized as $q_\phi (v,w,z|x,y)  = p_\beta (v | z,w,y)q_{\phi_w}(w|x,y)q_{\phi_z}(z|x)$ where $\phi = (\phi_x, \phi_w)$. We parametrize variational factors with networks  $\phi$ that output mean and diagonal covariance of variational distributions and specify their form to be Gaussian posteriors:
\begin{align}
(z|x) & \sim \mathcal{N}  \left( \mu(x; \phi_z), \text{diag} \left( \sigma^2(x; \phi_z) \right) \right)  \\
(w|x,y) & \sim \mathcal{N}  \left( \mu(x,y; \phi_w), \text{diag} \left( \sigma^2(x,y; \phi_w) \right) \right). 
\end{align}
There is a $p_\beta$ factor in the $q_\phi$ factorization because the $p_\beta$ factor can be written in terms of generative factors, lowering the number of trainable parameters. Using Bayes', we can rewrite $p_\beta (v | z,w,y)$ as
\begin{align}
p_\beta (v | z,w,y) &= \frac{p_\beta (z|w,y,v) p(v|y)}{\sum_{v'} p_\beta (z|w,y,v') p(v'|y)}. \label{eq:v-post-bayes}
\end{align}
The details are provided in the supplementary material. Another benefit is that $p_\beta (v | z,w,y)$ can be computed for all $v$ with simply one forward pass. The GMVAE's ELBO is then given by
\begin{align}
&\mathcal{L} (K)  = \E_{q_\phi (v,w,z|x,y)} \left[ \log \frac{p_{\beta,\theta}(x,v,w,z|y)}{q_\phi (v,w,z|x,y)} \right] \\
&= \E_{q_{\phi_z} (z|x)}\left[ \log p_\theta (x|z) \right] \quad \text{(reconstruction)} \\
&- \mathbb{E}_{q_{\phi_w}(w|x,y) q_{\phi_z}(z|x)  } \left[ \log q_{\phi_z}(z|x)  - \sum_{j=1}^{K_c} p_\beta ( v=j|z, w, y)  \log p_\beta ( z|w, y, v=j)  \right] \begin{pmatrix}
\text{latent} \\
\text{covering}
\end{pmatrix} \\
&- KL(q_{\phi_w}(w|x,y) || p(w) ) \quad \text{($w$-prior)}\\
&- \mathbb{E}_{q_{\phi_w}(w|x,y) q_{\phi_z}(z|x)} \left[ KL( p_\beta ( v|z, w, y)  || p (v|y) ) \right] \quad \text{(subcluster $v$-prior)}.
\end{align}
Since $K=(K_1, K_2, ..., K_C)$ is user-defined, the ELBO dependence on $K$ is made explicit and used later in the analyses. The reconstruction term promotes a latent representation meaningful to reconstruct the samples. The latent covering term attempts to subcluster the latent representation based on classes. The $w$-prior and subcluster $v$-prior terms drive those posteriors closer to their respective priors.

\subsection{Modification of the ELBO: removing $v$-prior} \label{section:no-v}
In this subsection, we propose removing the $v$-prior term from the original ELBO to make GMVAE more amenable to open-set recognition for two reasons. First, minimizing the $v$-prior term $\mathbb{E}_{q_{\phi_w}(w|x,y) q_{\phi_z}(z|x)} \left[ KL( p_\beta ( v|z, w, y)  || p (v|y) ) \right]$ is in direct conflict with the goal of distinct subclustering within a class. Our goal is to create disjoint subclusters in a class's latent representation so as to further provide reconstruction more flexibility and alleviate the assumption that a class's embedding is a convex set. However, notice that the $v$-prior term is minimized when $p_\beta ( v|z, w, y) = p (v|y)$ for every $z$, $w$, and $y$. Combined with \eqref{eq:v-post-bayes} and a uniform $p(v|y)$, this in turn implies that $p_\beta (z|w,y,v=i) = p_\beta (z|w,y,v=j)$ for every $w$, $y$, $i$, and $j$. Equivalent generative model distributions leads to mode collapse in the latent subclusters due to the maximization of the latent covering term. Put differently, the $v$-prior term discourages one-hot subcluster $v$ posteriors. However, this is exactly what is needed to robustly identify subclusters.

Second, as proven in Proposition \ref{thm:quasi-nonincreasing} in \S \ref{section:analytical}, without the $v$-prior term the optimal GMVAE loss for $C=1$ is non-increasing with respect to $K$. This is an analytical result which provides a heuristic procedure for identifying the appropriate number of subclusters $K_c$ to use for each class. Given these two reasons, for all the experiments in \S \ref{section:experiments}, we used the following modified ELBO:
\begin{align}
\mathcal{L}_\text{no $v$-prior} (K)  &= \E_{q_{\phi_z} (z|x)}\left[ \log p_\theta (x|z) \right] - KL(q_{\phi_w}(w|x,y) || p(w) )  \\
&- \mathbb{E}_{q_{\phi_w}(w|x,y) q_{\phi_z}(z|x)  } \left[ \log q_{\phi_z}(z|x)  - \sum_{j=1}^{K_c} p_\beta ( v=j|z, w, y)  \log p_\beta ( z|w, y, v=j)  \right].
\end{align}
In a sense, it is as if we do not impose a prior on the subcluster distributions. While we could have also negated the $v$-prior term, simply removing it actually yields the best experimental results.

\subsection{Open-set classification algorithms} \label{section:algorithms}
With recent literature in open-set recognition, it has nearly become universal to model class-belongingness by fitting a Weibull distribution to the inlier distances between a class's latent representations and its centroid \cite{TowardsOpenSetDeepNetworks, inter-intra, crosr}.
Indeed, the benchmark method CROSR \cite{crosr} achieves state-of-the-art accuracies through this EVT framework. However our experiments demonstrate that two much simpler algorithms can significantly outperform CROSR's EVT-based classification algorithm. While fitting a distribution to the inlier distances may be an effective way to conform a hypersphere-density around a class's centroid, we believe it is much too sensitive to these inliers. If a class has samples whose latent representations are ``misclassified'' and far away from its centroid, then the resulting distribution fit will be extremely skewed and render inaccurate predictions. This possible negative effect is severely magnified for embeddings that do not optimize for low intra-spread within each class. For instance, CROSR's embedding is composed of the closed-set, softmax classifier's activation vector; this encourages elements of that vector to tend towards positive and negative infinity. 

Next we present the two simple open-set classification algorithms we implemented. While GMVAE outputs a Gaussian distribution in latent space, we simply choose the mean $\mu(x; \phi_z)$ as the effective latent representation. Algorithm \ref{alg:nc-d} is derived from the so-called outlier score from \cite{inter-intra} but is most aptly described as nearest centroid thresholding on distance to the nearest centroid. This algorithm is modified to incorporate multiple subclusters per class. 

\begin{algorithm}[h!] 
Input: Training samples $\mathcal{X}_c$ for each known class $c=1,2,...,C$ and test sample $\widehat{x}$ \\
1. For each class $c$, compute $K_c$ centroids of $\mu(\mathcal{X}_c; \phi_z)$ using $k$-means clustering. Denote centroid $\overline{z}_{ck}$ as $k$-th centroid of class $c$. \\
2. Let $(c^*, k^*) = \argmin_{c, k} || \mu(\widehat{x}; \phi_z) - \overline{z}_{ck} ||_2$  and $d = \min_{c, k} || \mu(\widehat{x}; \phi_z) - \overline{z}_{ck} ||_2$ \\
3. If $d <  \tau$, predict class as  $c^*$; else, predict class as unknown $C+1$
\caption{Nearest centroid thresholding on distance to the nearest centroid}
\label{alg:nc-d}
\end{algorithm} 

Experimental results show that thresholding on distance to the nearest centroid more robustly fits a hypersphere around the respective centroid. However, a similar shortcoming shared with CROSR's EVT method is that distance is a rotationally symmetric measure. It does not include any sense of orientation. We stand to reason that in any nearest centroid-based algorithm, the open space between centroids poses the most risk from an open-set classification standpoint. This leads into the second algorithm which utilizes a novel threshold on an ``uncertainty'' quantity $U$. We define $U$ as the ratio between the distance to the nearest centroid to the average distance to all other centroids. So if $U=1$, the test sample's latent representation is equidistant from all centroids which can be interpreted as unclassifiable. If $U=0$, the test sample's latent representation is exactly a centroid meaning there is no ambiguity in classification. In this way, Algorithm \ref{alg:nc-u} includes a notion of orientation between centroids as $U$ penalizes the open space directly between centroids more heavily. This is reminiscent of the nearest neighbors distance ratio of \cite{nearest-neighbor}.

\begin{algorithm}[h!]
Input: Training samples $\mathcal{X}_c$ for each known class $c=1,2,...,C$ and test sample $\widehat{x}$ \\
1. For each class $c$, compute $K_c$ centroids of $\mu(\mathcal{X}_c; \phi_z)$ using $k$-means clustering. Denote centroid $\overline{z}_{ck}$ as $k$-th centroid of class $c$. \\
2. Let $(c^*, k^*) = \argmin_{c, k} || \mu(\widehat{x}; \phi_z) - \overline{z}_{ck} ||_2$, $N = \sum_{c=1}^C K_c$, and
\begin{align}
U = \frac{\min_{c, k} || \mu(\widehat{x}; \phi_z) - \overline{z}_{ck} ||_2}{\frac{1}{N-1} \sum_{(c, k)\neq(c^*, k^*)} || \mu(\widehat{x}; \phi_z) - \overline{z}_{ck} ||_2}
\end{align}
3. If $U < \tau$, predict class as  $c^*$; else, predict class as unknown $C+1$
\caption{Nearest centroid thresholding on uncertainty $U$}
\label{alg:nc-u}
\end{algorithm}

\section{Identifying the number of subclusters in each class} \label{section:analytical}
Since the number of subclusters in each class is user-defined, identifying the appropriate number is critical for model usage. A natural procedure that immediately arises is to iteratively apply GMVAE to each class's data alone for an increasing number of subclusters $K_c$. Given the reconstruction and clustering objectives, the empirical model losses should naturally inform us of the optimal number of subclusters. This is akin to increasing $k$ in $k$-means clustering and studying the resulting inertia plot. To this end, in this section we first present analytical results regarding the effect of $K = K_1$ on the optimal $C=1$ (single class) GMVAE loss. This then leads to our heuristic procedure for identifying the ideal number of subclusters in each class.  

With two unrestrictive neural network assumptions, we are able to prove two propositions regarding the effect of $K$ on the optimal GMVAE loss. The assumptions and proofs can be found in the supplementary material. The first proposition demonstrates that when there truly is only one subcluster within a class, and we know its distribution, then the optimal loss is constant with respect to $K$. Since $C=1$, we write $x$ instead of $(x,y)$.
\begin{proposition}
Let us assume that $x\in\mathcal{X}$ is distributed as $x\sim p_\text{data} = \mathcal{B}(\mu_x)$, $C=1$, and Assumption 1 holds. Then the optimal GMVAE loss is constant with respect to $K$. In fact, we have that $\min -\E_\mathcal{X} [\mathcal{L} (K)] =  -\E_\mathcal{X} [\log p_\text{data}]$ for every $K\geq1$ and a globally optimal solution reads 
\begin{align}
\mu(x; \phi_z^*) &= \mu_{c=1, k}(w;\beta^*) = \mu_z, \quad \sigma^2(x; \phi_z^*) = \sigma^2_{c=1, k}(w;\beta^*) = \sigma^2_z \\
\mu(x,y; \phi_w^*)& = \vec{0}, \quad \sigma^2(x,y; \phi_w^*) = \vec{1}, \quad \mu(z; \theta^*) = \mu_x
\end{align}
for any constant vectors $\mu_z, \sigma_z$.
\end{proposition}
The second proposition makes no data assumptions but shows that the optimal loss is quasi-non-increasing with respect to $K$. However, it does demonstrates the uniform lower-boundedness of sequential differences.
\begin{proposition} \label{thm:quasi-nonincreasing}
Let us assume $C=1$, Assumptions 1 and 2 hold, and that $p(v|y=1)$ is uniform in the appropriate dimension. We have $
\min\left\{ -\E_\mathcal{X} [\mathcal{L}(K; \phi_z, \phi_w, \beta, \theta)] \right\} - \min \left\{ -\E_\mathcal{X}[\mathcal{L}(K+1; \phi_z, \phi_w, \beta, \theta)]\right\} \geq \epsilon_K
$
where $-\log 2 \leq \log(K/(K+1)) \leq \epsilon_K$ for all $K$.
\end{proposition}
These proofs do not inform us on the transient dynamics of training nor even reaching the global optimum. As such, in the following experimental results section, we apply these propositions in practice by comparing the latent covering loss given reconstruction loss for each $K\geq 1$. This answers: How well does $K$ subclusters ``cover'' the embedding for a given reconstruction level? When the latent covering loss's decreases begins to diminish, then it is an indication that additional subclusters were only marginally beneficial and perhaps should not be included. 

\section{Experimental results} \label{section:experiments}
The experimental results demonstrate several findings. First, EVT-based open-set classification may not be appropriate as simple nearest centroid procedures consistently beat it. Second, even without the added benefit of subclustering, GMVAE for $K  = \vec{1}$ often leads to a latent representation more amenable for open-set recognition compared to CROSR. Finally, subclustering within classes represents a means of bolstering dual supervised-reconstruction embeddings. 

Each dataset has the following composition. The training data has only labeled samples from the $C$ classes. The validation set has samples from the same $C$ classes and samples from additional $Q$ classes which are all treated as class $C+1$. The validation set is used to determine the threshold $\tau$. Finally, the test set has the same distribution as the validation set.

For each of the experiments below, we perform an ablation study. Four combinations of model and classification algorithms were applied: (i) CROSR with CROSR's EVT (CROSR+EVT), (ii) CROSR with Algorithm \ref{alg:nc-d} (CROSR+NC-D), (iii) GMVAE with Algorithm \ref{alg:nc-d} (GMVAE+NC-D), and (iv) GMVAE with Algorithm \ref{alg:nc-u} (GMVAE+NC-U). CROSR+NC-D and GMVAE+NC-D are meant to directly compare latent representations' amenability to open-set recognition. We did not study CROSR with Algorithm \ref{alg:nc-u} because our ``uncertainty'' measure is really a proxy for confidence and it has been shown that it is erroneous to equate softmax classifiers with confidence \cite{confidence}. Correctly adapting ``uncertainty'' to CROSR is out of this paper's scope. For each combination, we calculate the optimal macro-averaged F1 scores (optimizing threshold $\tau$ on the validation set) for an increasing number $Q$ of unknown classes (and samples). However, for actual inference in model serving, we would not have any prior knowledge on the known and unknown composition of the test set. Accordingly, we also compute the F1 scores using the mean of those optimal thresholds. The first two experiments are for $K=\vec{1}$ and in the last two, we manufacture classes with multiple subclusters to apply $K = (2,2)$.

We optimize over the training set using Adam until the loss, evaluated on the known validation set, plateaus. For the MNIST and Fashion MNIST datasets (grayscale images), the reconstruction distribution used was the unnormalized, continuous Bernoulli distribution. For the CIFAR-10 dataset (RGB images), a truncated $[0,1]$ Gaussian models the reconstruction. The latent space dimension of $z$ equals 10, 50, 5, and 20 for the four experiments. We will publish our code upon acceptance of this paper.

\subsection{Fashion MNIST withholding 4 classes}
The six known classes are t-shirts/tops, trousers, pullovers, dresses, coats, and shirts, while the four unknown classes are sandals, sneakers, ankle boots, and bags. Fashion MNIST's standard training set is randomly split into the validation set (6,000 samples of known classes and 4,000 samples of unknown classes) and training set (30,000 samples). Fashion MNIST's standard testing set (10,000 samples) is kept the same. We use the same CROSR network architecture as \cite{crosr} for their MNIST experiment. The $K=\vec{1}$ GMVAE network architectures are given in Table \ref{tab:fashion-networks}. The $\theta$ network is always the mirrored $\phi_z$ network. Convolution parameters are denoted by ``$\langle$number of layers$\rangle$ conv$\langle$kernel size$\rangle$-$\langle$number of channels$\rangle$'' and max pooling strides are denoted by ``maxpool-$\langle$stride$\rangle$.'' ReLU activations follow each weight layer.

\begin{table}[h!]
\centering
\caption{Network architectures for $K=\vec{1}$ GMVAE used in Fashion MNIST experiment.}
\label{tab:fashion-networks}
\begin{tabular}{ ccc }   % top level tables, with 3 columns
$\phi_z$ & $\phi_w$  & $\beta$ \\  
% leftmost table of the top level table
\begin{tabular}[t]{ c } 
\hline
Input: $x$ \\
2 conv3-20 \\ 
maxpool-2 \\
2 conv3-50 \\
maxpool-2 \\
FC-500 \\
FC-100 \\
FC-20 ($2\times \text{dim}(z)$) \\
\hline
\end{tabular} &  % starting center sub table
% table 2
\begin{tabular}[t]{ c} 
\hline
Input: $x, y$ \\
1 conv3-10 on $x$ only \\
maxpool-4 \\
Concatenate with $y$ \\
FC-20 ($2\times \text{dim}(w)$) \\
\hline
\end{tabular} & % starting rightmost sub table
% table 3
\begin{tabular}[t]{ c} 
\hline
Input: $w$ \\
FC-50 \\
FC-50 \\
FC-120 ($6\times2\times \text{dim}(z)$) \\ 
\hline
\end{tabular} \\
\end{tabular}
\end{table}

F1 scores are plotted in Figure \ref{fig:fashion-f1}. On the left are the optimal threshold F1 scores versus the number of unknown classes $Q$. While GMVAE is not as accurate in the closed-set regime, it outperforms CROSR as $Q$ increases. CROSR's open-set accuracies, in turn, quickly diminish as $Q$ increases, CROSR+EVT in particular. On the right is the mean threshold F1 scores where GMVAE's F1 scores are still more robust to increasing $Q$. For $Q \geq 1$ and using the mean thresholds, GMVAE+NC-U F1 scores are on average 3.9\% greater than those of CROSR+NC-D. 

\begin{figure}[h!]
\centerline{
  \includegraphics[width=0.38\linewidth]{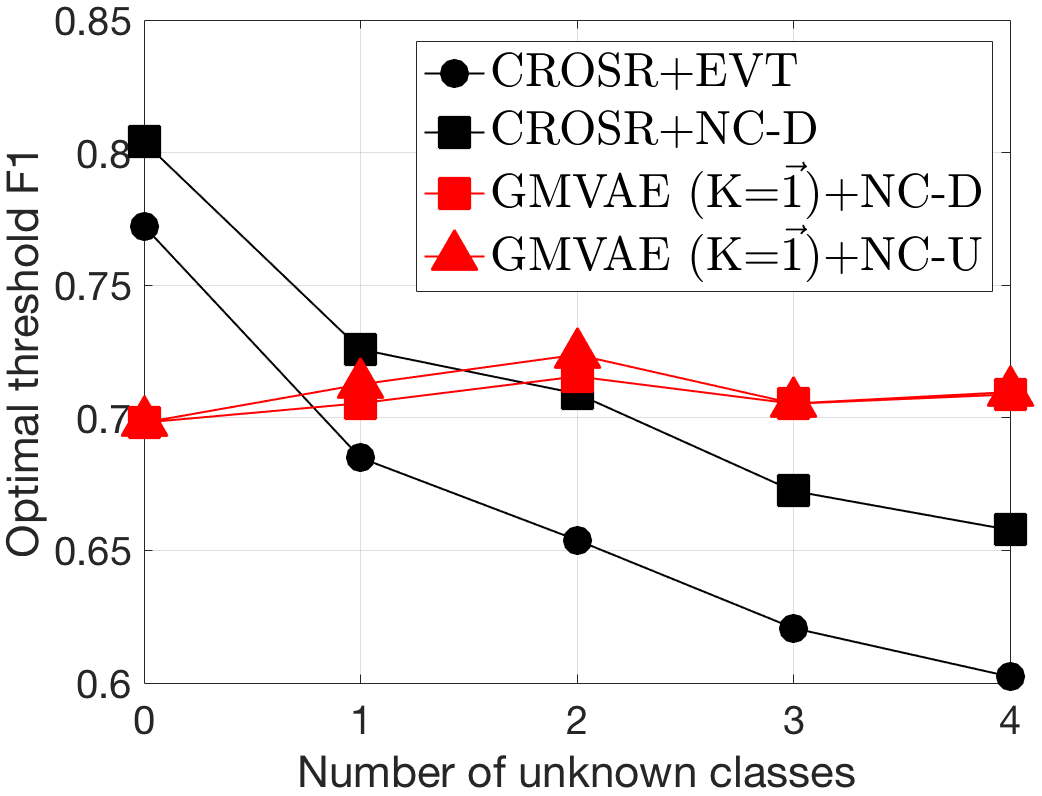} 
  \includegraphics[width=0.38\linewidth]{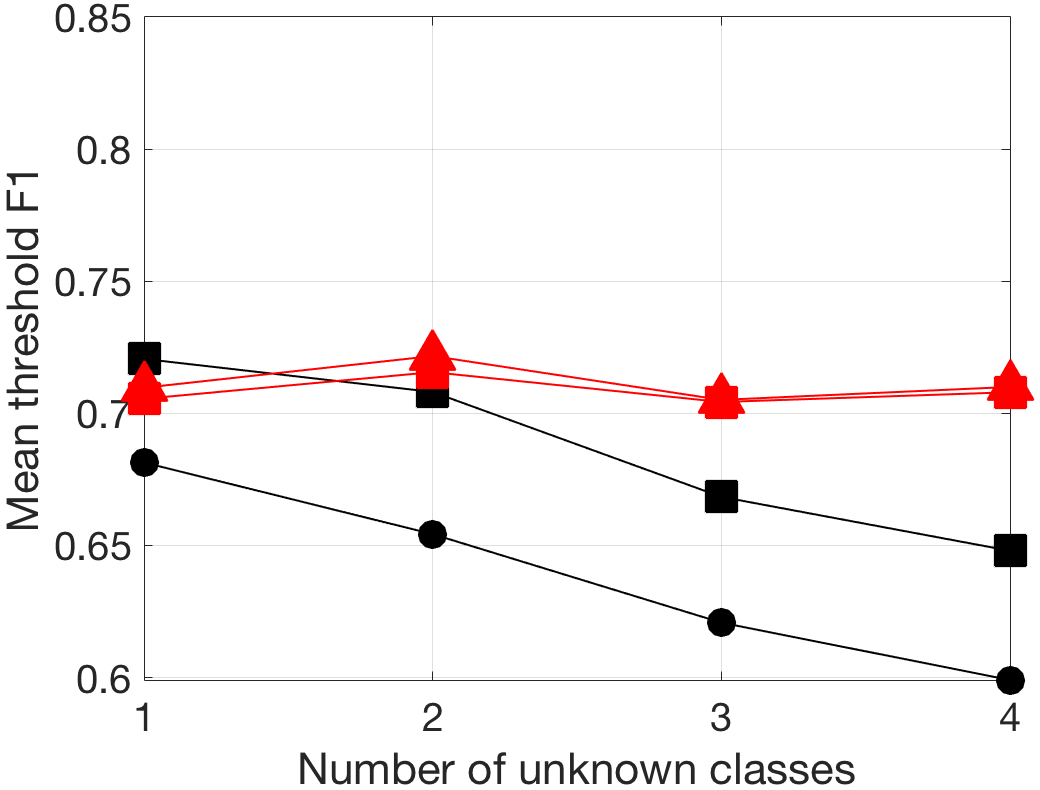}
}
\caption{Fashion MNIST open-set test F1 scores.}
\label{fig:fashion-f1}
\end{figure}

\subsection{CIFAR-10 withholding 4 classes}
The six known classes are airplanes, automobiles, birds, cats, deer, and dogs. The four unknown classes are frogs, horses, ships, and trucks. CIFAR-10's standard training set is randomly split into the validation set (6,000 samples of known classes and 4,000 samples of unknown classes) and training set (24,000 samples). CIFAR-10's standard testing set (10,000 samples) is kept the same. For both CIFAR-10 experiments, we use the same CROSR architecture as \cite{crosr} for their CIFAR-10 experiment. The $K=\vec{1}$ GMVAE network architectures are given in Table \ref{tab:cifar10-k1-networks}. For both CIFAR-10 experiments, the $\phi_z$ network is pretrained on the known classes and those weights are then frozen.

\begin{table}[h!]
\centering
\caption{Network architectures for $K=\vec{1}$ GMVAE used in the first CIFAR-10 experiment.}
\label{tab:cifar10-k1-networks}
\begin{tabular}{ ccc }   % top level tables, with 3 columns
$\phi_z$ & $\phi_w$  & $\beta$ \\  
% leftmost table of the top level table
\begin{tabular}[t]{ c } 
\hline
Input: $x$ \\
2 conv3-64 \\ 
maxpool-2 \\
2 conv3-128 \\
maxpool-2 \\
4 conv3-256 \\
maxpool-2 \\
FC-1000 \\
FC-500 \\
FC-100 ($2\times \text{dim}(z)$) \\
\hline
\end{tabular} &  % starting center sub table
% table 2
\begin{tabular}[t]{ c} 
\hline
Input: $x,y$ \\
1 conv3-10 on $x$ only \\
maxpool-4 \\
Concatenate with $y$ \\
FC-100 ($2\times \text{dim}(w)$) \\
\hline
\end{tabular} & % starting rightmost sub table
% table 3
\begin{tabular}[t]{ c} 
\hline
Input: $w$ \\
FC-100 \\
FC-100 \\
FC-600 ($6\times2\times \text{dim}(z)$) \\ 
\hline
\end{tabular} \\
\end{tabular}
\end{table}

F1 scores are plotted in Figure \ref{fig:cifar10-k1-f1}. GMVAE consistently beats CROSR and again CROSR+EVT performs worst. Algorithm \ref{alg:nc-u} augments GMVAE and we deduce this is because unknown CIFAR-10 samples are more difficult to distinguish and thus more likely to be embedded to the interior of known latent clusters where ``uncertainty'' has more influence. For $Q \geq 1$ and using the mean thresholds, GMVAE+NC-U F1 scores are on average 61.3\% greater than those of CROSR+NC-D. 

\begin{figure}[h!]
\centerline{
  \includegraphics[width=0.38\linewidth]{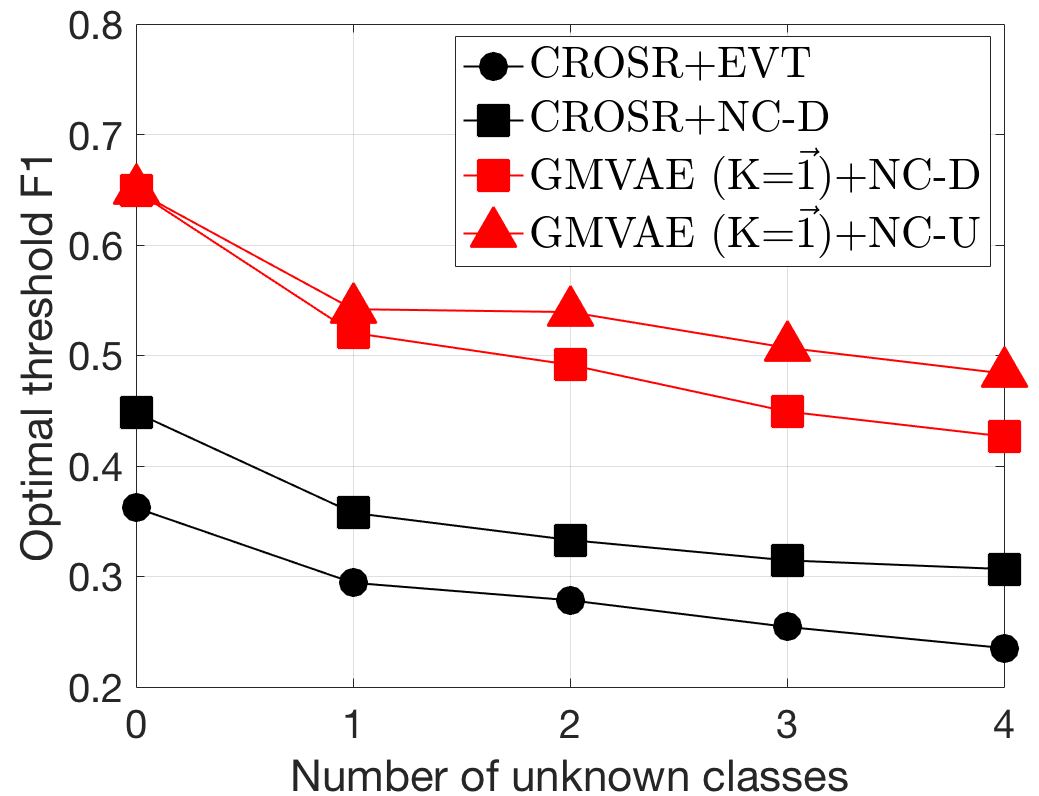} 
  \includegraphics[width=0.38\linewidth]{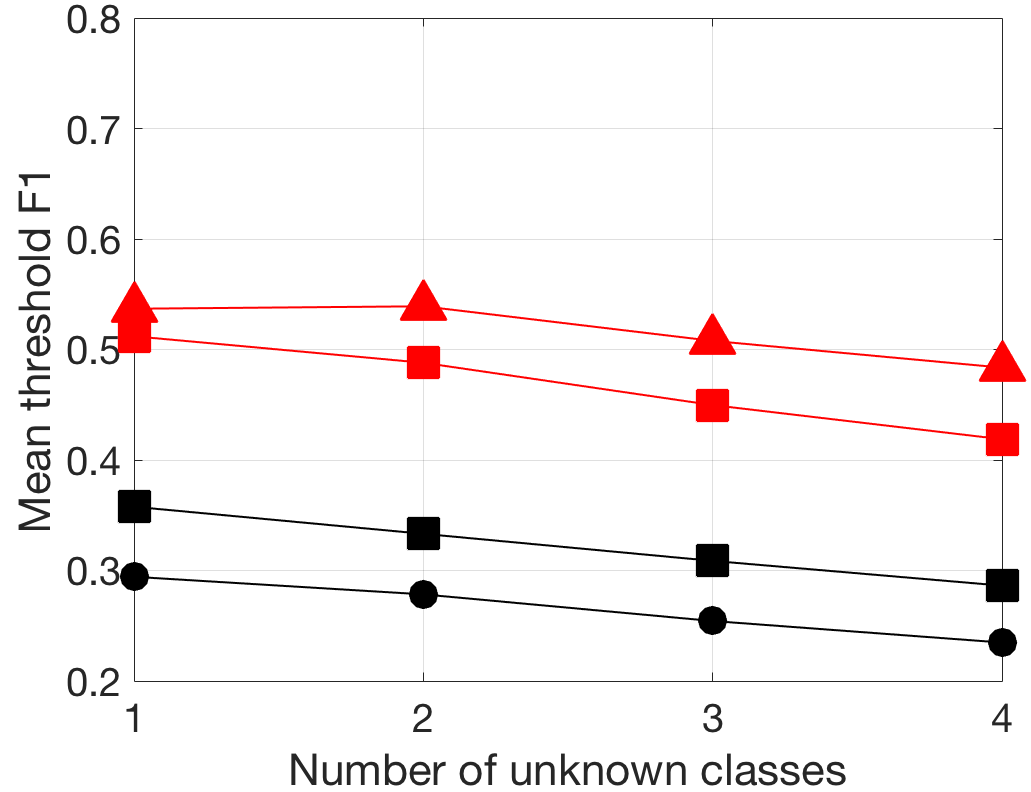}
}
\caption{$K=\vec{1}$ CIFAR-10 open-set test F1 scores.}
\label{fig:cifar10-k1-f1}
\end{figure} 
We present a t-SNE plot of the CROSR latent representation components in Figure \ref{fig:cifar10-k1-crosr-tsne} to bring into question the explicit use of classifier activation vectors in an open-set recognition embedding. We see that the reconstruction latent variable $z$ does little to cluster the known classes and so open-set classification is dominated by the known classifier's activation vector $y$. We believe this to be the underlying reason why CROSR's F1 scores in Figures \ref{fig:cifar10-k1-f1} and \ref{fig:cifar10-k2-f1} are so poor.

\begin{figure}[h!]
\centerline{
  \includegraphics[width=0.76\linewidth]{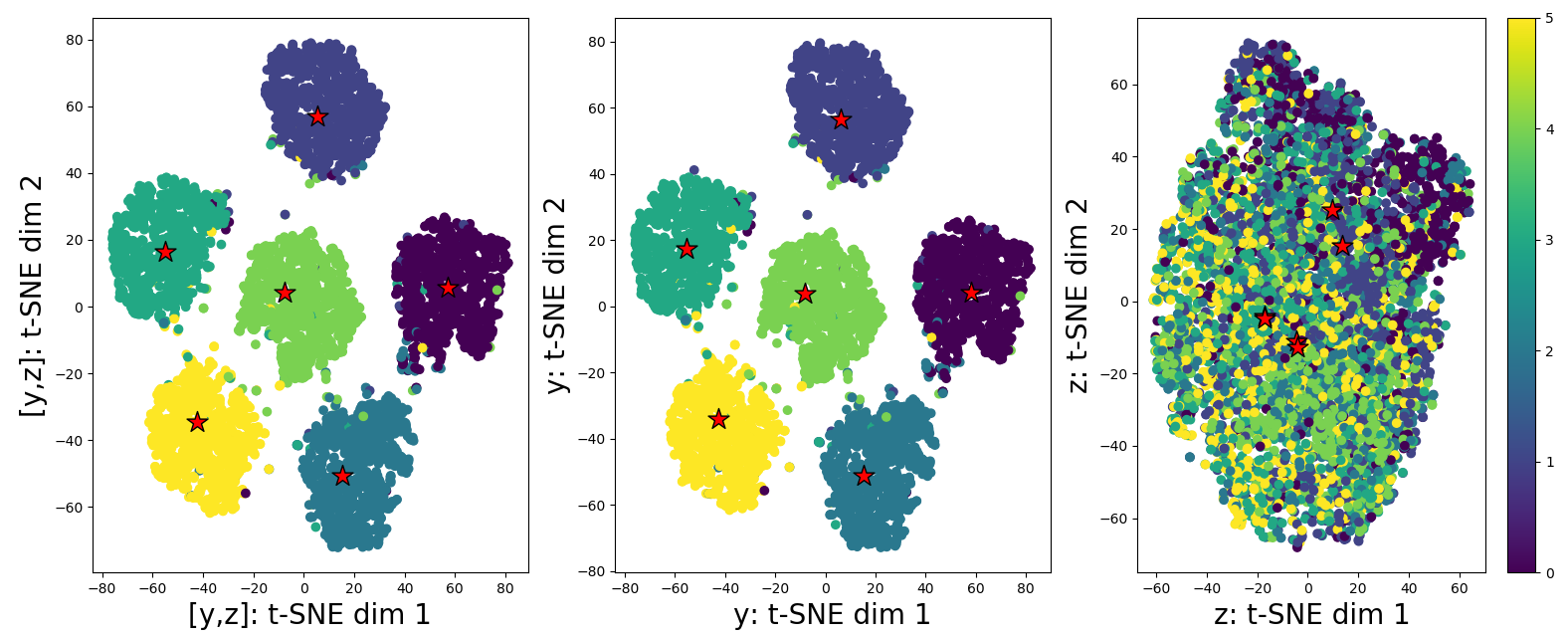} 
}
\caption{t-SNE plot of (left) both components $[y,z]$, (center) only $y$, and (right) only $z$ of CROSR's training latent representations for the first CIFAR-10 experiment. Stars are the respective component's class centroids.}
\label{fig:cifar10-k1-crosr-tsne}
\end{figure} 

In contrast to CROSR, GMVAE's latent representation $\mu(x; \phi_z)$ in Figure \ref{fig:cifar10-k1-gmvae-tsne} separates classes better (in comparison to the right figure in Figure \ref{fig:cifar10-k1-crosr-tsne}). GMVAE's embedding is able to effectively capture both class and reconstruction information simultaneously, leading to more amenable open-set recognition. As CIFAR-10 images are highly hetergeneous within classes, we expect class overlap from reconstruction.

\begin{figure}[h!]
\centerline{
  \includegraphics[width=0.38\linewidth]{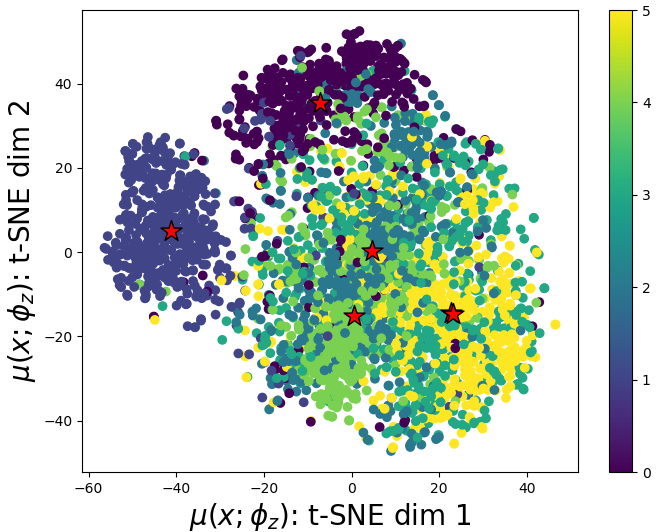} 
}
\caption{t-SNE plot of $\mu(x; \phi_z)$ of GMVAE's training latent representations for the first CIFAR-10 experiment. Stars are the respective component's class centroids.}
\label{fig:cifar10-k1-gmvae-tsne}
\end{figure} 

\subsection{MNIST with ``even'' and ``odd'' classes}
The two known classes are ``even,'' comprised of digits 0 and 2, and ``odd,'' comprised of digits 1 and 3. The six unknown classes are digits 4 and greater. MNIST's standard training set is randomly split into the validation set (4,000 samples of known classes and 6,000 samples of unknown classes) and training set (about 18,000 samples). MNIST's standard testing set (10,000 samples) is kept the same. We use the same CROSR architecture as \cite{crosr} for their MNIST experiment. The GMVAE network architectures are given in Table \ref{tab:mnist-networks}. 

\begin{table}[h!]
\centering
\caption{Network architectures for GMVAE used in the ``even'' and ``odd'' MNIST experiment.}
\label{tab:mnist-networks}
\begin{tabular}{ ccc }   % top level tables, with 3 columns
$\phi_z$ & $\phi_w$  & $\beta$ \\  
% leftmost table of the top level table
\begin{tabular}[t]{ c } 
\hline
Input: $x$ \\
1 conv3-20 \\ 
maxpool-2 \\
1 conv3-50 \\
maxpool-2 \\
FC-500 \\
FC-100 \\
FC-10 ($2\times \text{dim}(z)$) \\
\hline
\end{tabular} &  % starting center sub table
% table 2
\begin{tabular}[t]{ c} 
\hline
Input: $x,y$ \\
1 conv3-10 on $x$ only \\
maxpool-4 \\
Concatenate with $y$ \\
FC-10 ($2\times \text{dim}(w)$) \\
\hline
\end{tabular} & % starting rightmost sub table
% table 3
\begin{tabular}[t]{ c} 
\hline
Input: $w$ \\
FC-20 \\
FC-20 \\
FC-$\left( 2\times \sum_c K_c \times \text{dim}(z) \right)$ \\ 
\hline
\end{tabular} \\
\end{tabular}
\end{table}

This is a clearcut example where each class has two subclusters. To determine that $K=(2,2)$ is indeed the optimal GMVAE selection, we implement the procedure in \S \ref{section:analytical} in Figure \ref{fig:mnist-subcluster}. On the left, the mean difference between the $K=1$ and $K=2$ latent covering loss is 0.86 while the mean difference between $K=2$ and $K=3$ is 0.22. This is indicative of two true subclusters within ``even.'' Similarly on the right, the mean difference between $K=1$ and $K=2$ latent covering loss is 1.23 while the mean  difference between $K=2$ and $K=3$ is -0.09. This is again indicative of two true subclusters within ``odd.'' For these plots, the early epochs are truncated.

\begin{figure}[h!]
\centerline{
  \includegraphics[width=0.38\linewidth]{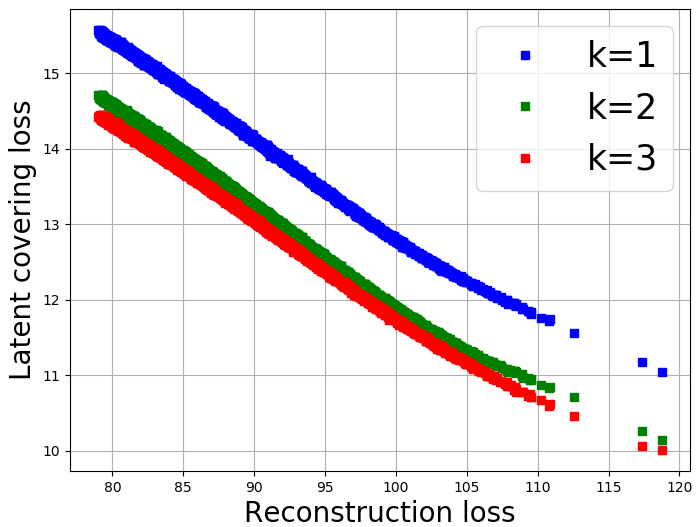}
  \includegraphics[width=0.38\linewidth]{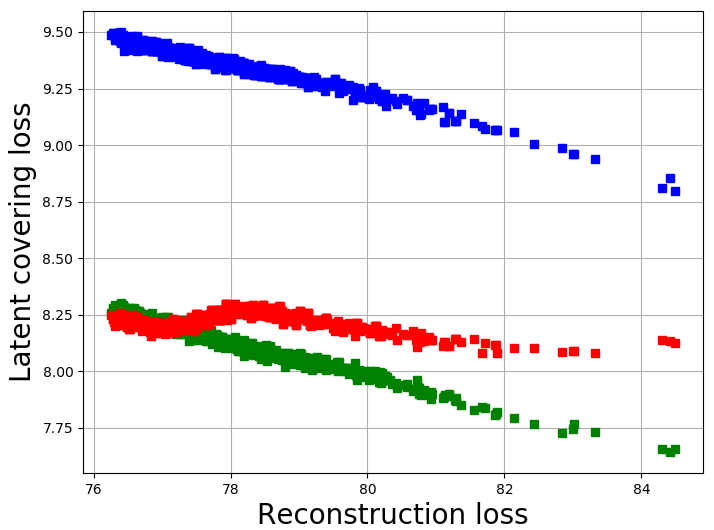} 
}
\caption{The latent covering loss plotted against reconstruction loss for increasing $K$ for the (left) ``even'' and (right) ``odd'' classes of MNIST.}
\label{fig:mnist-subcluster}
\end{figure} 

F1 scores are plotted in Figure \ref{fig:mnist-f1}. Similar to the Fashion MNIST experiment, GMVAE+NC-D begins to outperform CROSR+NC-D when the number of unknown classes $Q \geq 3$. However slightly, CROSR+EVT again performs worst. There is a significant increase in GMVAE open-set accuracy and robustness to increasing $Q$ from utilizing the ``uncertainty'' threshold. This algorithm complements the use of class subclusters as unknown classes' latent representations are strategically more likely embedded in the open space between centroids where $U$ is larger. For $Q \geq 1$ and using the mean thresholds, GMVAE+NC-U F1 scores are on average 19.9\% greater than those of CROSR+NC-D.

\begin{figure}[h!]
\centerline{
  \includegraphics[width=0.38\linewidth]{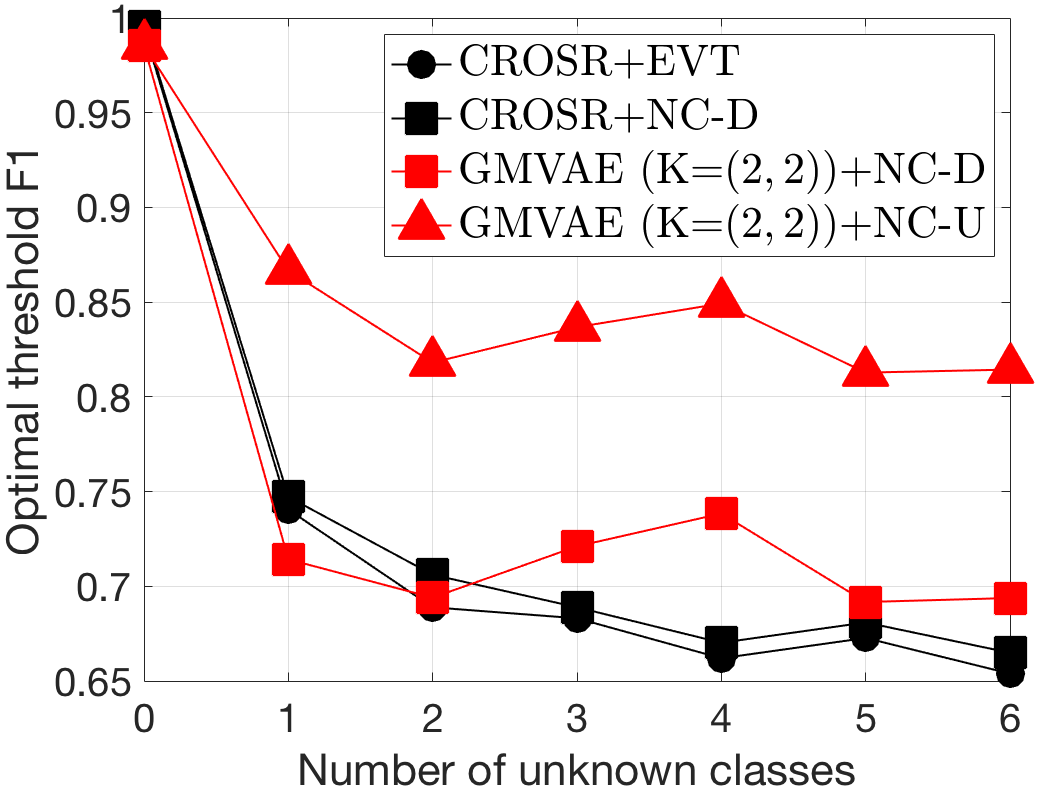} 
  \includegraphics[width=0.38\linewidth]{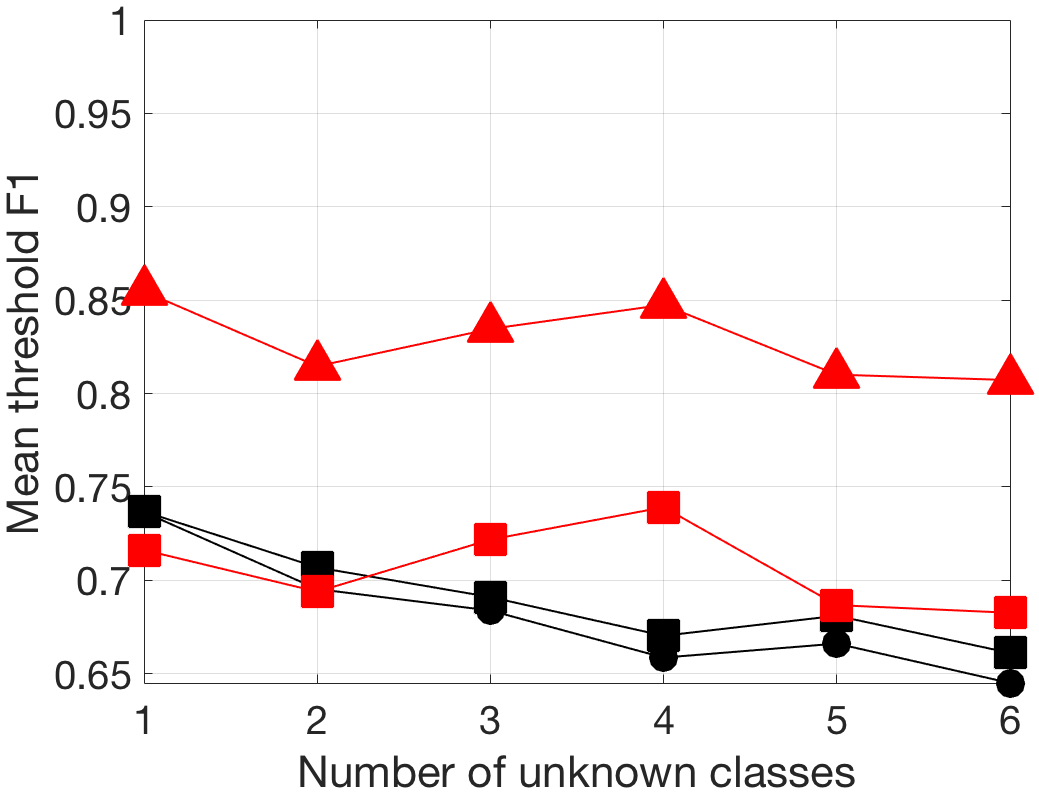}
}
\caption{``Even'' and ``odd'' MNIST open-set test F1 scores.}
\label{fig:mnist-f1}
\end{figure} 

\subsection{CIFAR-10 with ``animals'' and ``vehicles'' classes}
The two known classes are ``animals,'' comprised of cats and dogs, and ``vehicles,'' comprised of cars and trucks. The unknown classes are the other 6 classes. CIFAR-10's standard training set is randomly split into the validation set (4,000 samples of known classes and 6,000 samples of unknown classes) and training set (16,000 samples). CIFAR-10's standard testing set (10,000 samples) is kept the same. The GMVAE network architectures are given in Table \ref{tab:cifar10-k2-networks}. 

\begin{table}[h!]
\centering
\caption{Network architectures for GMVAE used in the second CIFAR-10 experiment with ``animals'' and ``vehicles''.}
\label{tab:cifar10-k2-networks}
\begin{tabular}{ ccc }   % top level tables, with 3 columns
$\phi_z$ & $\phi_w$  & $\beta$ \\  
% leftmost table of the top level table
\begin{tabular}[t]{ c } 
\hline
Input: $x$ \\
2 conv3-64 \\ 
maxpool-2 \\
2 conv3-128 \\
maxpool-2 \\
4 conv3-256 \\
maxpool-2 \\
FC-1000 \\
FC-500 \\
FC-40 ($2\times \text{dim}(z)$) \\
\hline
\end{tabular} &  % starting center sub table
% table 2
\begin{tabular}[t]{ c} 
\hline
Input: $x,y$ \\
1 conv3-10 on $x$ only \\
maxpool-4 \\
Concatenate with $y$ \\
FC-40 ($2\times \text{dim}(w)$) \\
\hline
\end{tabular} & % starting rightmost sub table
% table 3
\begin{tabular}[t]{ c} 
\hline
Input: $w$ \\
FC-50 \\
FC-50 \\
FC-$\left( 2\times \sum_c K_c \times \text{dim}(z) \right)$ \\ 
\hline
\end{tabular} \\
\end{tabular}
\end{table}

This is another clearcut example where each class has two subclusters. We again implement the procedure in \S 4 to determine that $K=(2,2)$ is indeed the optimal GMVAE selection and show the results in Figure \ref{fig:cifar10-k2-subcluster}. On the left, the mean difference between $K=1$ and $K=2$ latent covering loss is 1.31 while the mean difference between $K=2$ and $K=3$ is -1.47. This is indicative of two true subclusters within ``animals.'' On the right, the mean difference between $K=1$ and $K=2$ latent covering loss is 0.82 while the mean difference between $K=2$ and $K=3$ is 0.5. This is moderately indicative of two true subclusters within ``vehicles.''

\begin{figure}[h!]
\centerline{
  \includegraphics[width=0.38\linewidth]{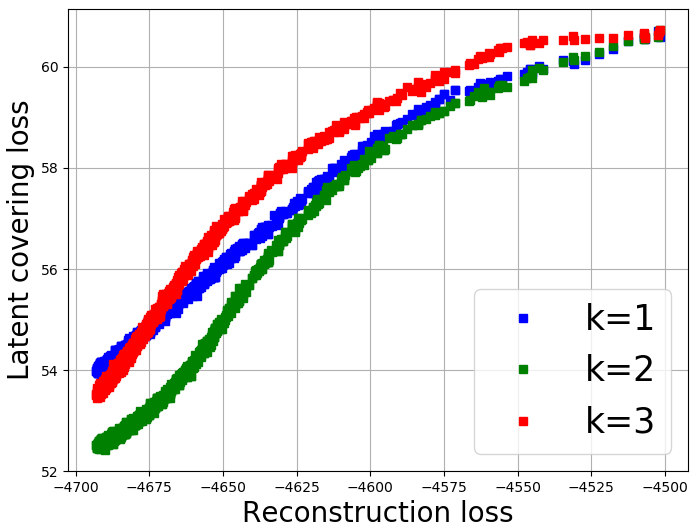} 
  \includegraphics[width=0.38\linewidth]{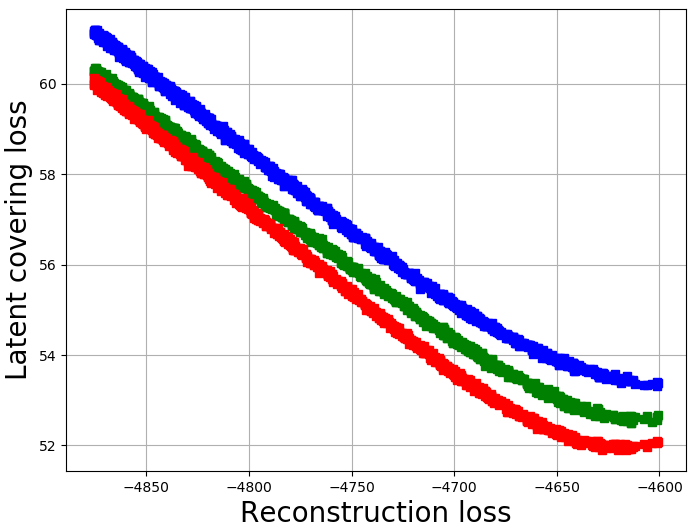} 
}
\caption{The latent covering loss plotted against reconstruction loss for increasing $K$ for the (left) ``animals'' and (right) ``vehicles'' classes of CIFAR-10.}
\label{fig:cifar10-k2-subcluster}
\end{figure} 

F1 scores are plotted in Figure \ref{fig:cifar10-k2-f1}. Discussed in \S \ref{section:algorithms}, as a result of CROSR's softmax classifier, the centroids are not representative and thus its closed-set classification suffers. While CROSR+EVT eventually recovers to near similar performance as GMVAE+NC-D for $Q \geq 4$, the difference is stark for a low number of unknown classes. Again, because of the class subclusters, the ``uncertainty'' threshold provides a significant increase in open-set recognition capability. For $Q \geq 1$ and using the mean thresholds, GMVAE+NC-U F1 scores are on average 33\% greater than those of CROSR+EVT.

\begin{figure}[h!]
\centerline{
  \includegraphics[width=0.38\linewidth]{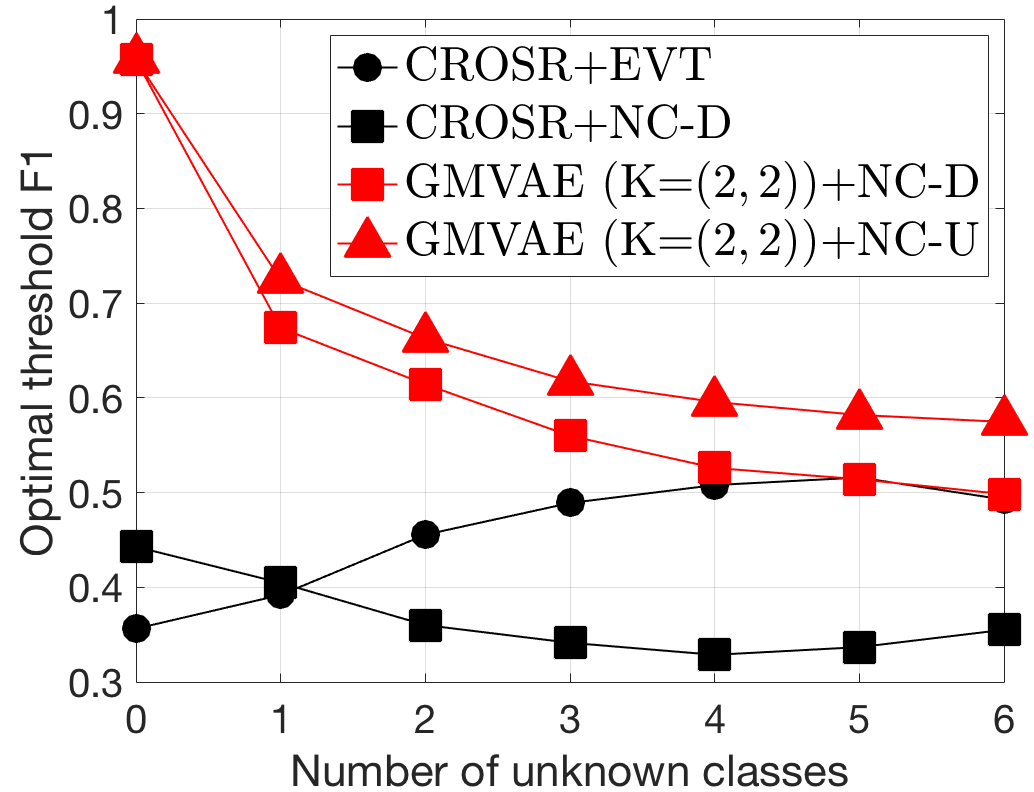} 
  \includegraphics[width=0.38\linewidth]{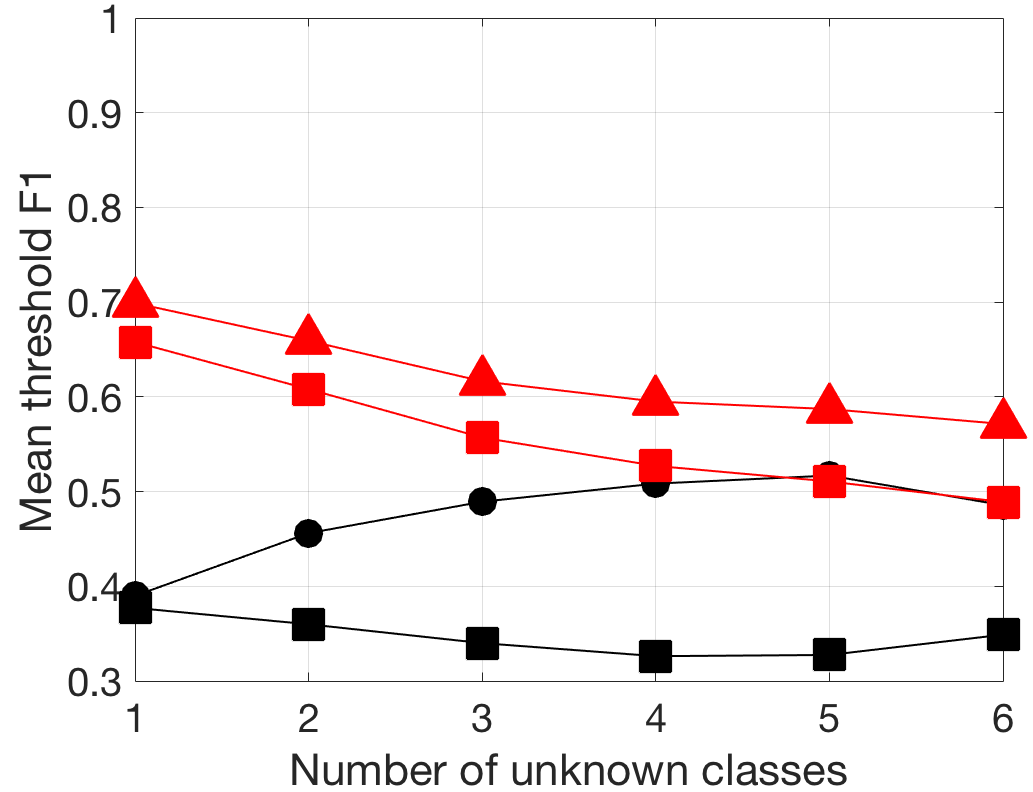}
}
\caption{$K=(2,2)$ CIFAR-10 open-set test F1 scores.}
\label{fig:cifar10-k2-f1}
\end{figure} 

\begin{ack}
The work of the first author is supported by the Predoctoral Training Program in Biomedical Data Driven Discovery (BD3) at Northwestern University (National Library of Medicine Grant 5T32LM012203). The work of the second author is supported in part by NIH Grant R21LM012618.
\end{ack}

\bibliography{open-set}

\end{document}

\maketitle

\section{Neural network assumptions}
We call a neural network $f_\tau$ an $n$-headed neural network if
\begin{enumerate}
\item $f_\tau: \R^m \rightarrow \prod_{i=1}^n \R^s$, i.e. it maps $b$ to $(a_1, a_2, ..., a_n)$ with $a_i \in \R^s$,
\item for each $i$, $1\leq i \leq n$, we have $a_i = f_{\ell_i}^i \circ f_{\ell_i - 1}^i \circ ... \circ f_{t+1}^i \circ f_{t} \circ f_{t-1} \circ  ... \circ f_{1} (b)$ for  an integer $t$ not depending on $i$, $\ell_i \geq t+1$, and each $f_j, f_j^i$ is a typical neural network single layer parameterized by a matrix and a bias vector, and it includes an activation function. Vector $\tau$ corresponds to all these parameters.
\end{enumerate}

In GMVAE, neural networks corresponding to  $q_{\phi_z}, q_{\phi_w}$ are 2-headed neural networks (mean and covariance) with $\phi_z, \phi_w$ denoting all of the respective parameters. Probability $p_\theta$ is a 1 or 2-headed network with parameters $\theta$, and $p_\beta$ for  $\beta  = (\beta_{K_1}, \beta_{K_2}, ..., \beta_{K_C})$ consists of a $\left(2\sum_{c=1}^C K_c\right)$-headed neural network.

\begin{assumption}
In each network $q_{\phi_z}$, $q_{\phi_w}$, $p_\theta$, and $p_\beta$, the last layer in each head $f_{\ell_i}^i$ has an identity activation function.
\end{assumption}
\begin{assumption}
Neural network $p_{\beta'}$ for $\beta'  = (\beta_{K_1},..., \beta_{K_c+1}, ..., \beta_{K_C})$ consists of $p_\beta$ with simply two additional heads, while all other network architectures are the same.
\end{assumption}

\begin{lemma}\label{thm:any-constant}
Under Assumption 1 for an $n$-headed network, we have that given any $\overline{a} = (\overline{a}_1, ...., \overline{a}_n)$, there exists $\tau = \tau (\overline{a})$ such that $f_\tau (b) =  \overline{a}$ for every $b$.
\end{lemma}
\begin{proof}
Let $\overline{a} $ be given. We define $\tau$ to consist of 0 matrices and biases for each layer except $f_{\ell_i}^i$. In $f_{\ell_i}^i$, the matrix is 0 but the bias is $\overline{a}_i$. Since $f_{\ell_i}^i$ has the identity activation, it follows $f_\tau (b) = \overline{a}$ for every $b$.
\end{proof}

\section{Proof of Proposition 1}
\begin{proposition}
Let us assume that $x\in\mathcal{X}$ is distributed as $x\sim p_\text{data} = \mathcal{B}(\mu_x)$, $C=1$, and Assumption 1 holds. Then the optimal GMVAE loss is constant with respect to $K$. In fact, we have that $\min -\E_\mathcal{X} [\mathcal{L} (K)] =  -\E_\mathcal{X} [\log p_\text{data}]$ ~for every $K\geq1$ and a globally optimal solution reads 
\begin{align}
\left.
\begin{array}{ll}
\mu(x; \phi_z^*) &= \mu_{c=1, k}(w;\beta^*) = \mu_z \\
\sigma^2(x; \phi_z^*) &= \sigma^2_{c=1, k}(w;\beta^*) = \sigma^2_z \\
\mu(x,y; \phi_w^*)& = \vec{0}\\
\sigma^2(x,y; \phi_w^*)& = \vec{1}  \\
\mu(z; \theta^*) &= \mu_x  \\
\end{array}
\right\} \label{eq:trivial-solution}
\end{align}
for any constant vectors $\mu_z, \sigma_z$.
\end{proposition}
\begin{proof}
Note that $(\phi_z^*, \phi_w^*, \beta^*, \theta^*)$ exist due to Assumption 1 and Lemma \ref{thm:any-constant}. First, we show that $(\theta^*, \beta^*)$ given in \eqref{eq:trivial-solution} maximize the log likelihood $\E_\mathcal{X} \left[ \log p_{\theta, \beta} (x|y=1) \right]$ and results in $p_{\theta^*, \beta^*} (x|y=1) = p_\text{data}$. We have 
\begin{align}
KL(p_\text{data} || p_{\theta, \beta} (x|y=1)) = \E_\mathcal{X} \left[ \log p_\text{data} \right]-\E_\mathcal{X} \left[ \log p_{\theta, \beta} (x|y=1) \right]
\end{align}
and thus maximizing $\E_\mathcal{X} \left[ \log p_{\theta, \beta} (x|y=1) \right]$ is equivalent to minimizing $KL(p_\text{data} || p_{\theta, \beta} (x|y=1))$. The global minimum of $KL(p_\text{data} || p_{\theta, \beta} (x|y=1))$ is clearly when $p_\text{data} = p_{\theta, \beta} (x|y=1)$. This is indeed the case for $(\theta^*, \beta^*)$, since
\begin{align}
p_{\theta^*, \beta^*} (x|y=1) &= \int_{w,z,v}  p_{\beta^*, \theta^*} (x,v,w,z|y=1) dw dz dv \\
&= \int_{w,z,v}  p_{\theta^*} (x|z) p_{\beta^*} (z|w,y=1,v) p(v|y=1) p(w) dw dz dv \\
&= \int_{w,z,v}  p_\text{data}  p_{\beta^*} (z|w,y=1,v) p(v|y=1) p(w) dw dz dv \\
&= p_\text{data} \label{eq:generative-equals-data}
\end{align}
because of GMVAE's generative model factorization and \eqref{eq:trivial-solution}. Now we have 
\begin{align}
\E_\mathcal{X} \left[ \log p_\text{data}  \right] &= \E_\mathcal{X} \left[ \log p_{\theta^*, \beta^*} (x|y=1) \right] \\
& = \E_\mathcal{X} \left[  \E_{q_{\phi^*}(v,w,z | x,y=1)} \left[ \log  \frac{p_{\theta^*, \beta^*}  (x,z,w,v | y=1)}{q_{\phi^*}(v,w,z | x,y=1)}  \right] \right]  \\
& +\E_\mathcal{X} \left[\E_{q_{\phi^*}(v,w,z | x,y=1)} \left[  \log \frac{q_{\phi^*}(v,w,z | x,y=1)}{p_{\theta^*, \beta^*}  (z,w,v | x,y=1)}  \right] \right] \label{eq:VG}\\
&=\E_\mathcal{X} [ \mathcal{L} (K;\phi_z^*, \phi_w^*, \beta^*, \theta^*)] + \E_\mathcal{X} [\text{VG}(\phi_z^*, \phi_w^*, \beta^*, \theta^*)]\label{eq:l-vg}
\end{align}
where $\text{VG}(\phi_z^*, \phi_w^*, \beta^*, \theta^*)$ corresponds to \eqref{eq:VG}. We next show that $\text{VG}(\phi_z^*, \phi_w^*, \beta^*, \theta^*)=0$. This together with the facts that maximized $\E_\mathcal{X} [\mathcal{L} (K;\phi_z, \phi_w, \beta, \theta)]$ corresponds with minimized $\E_\mathcal{X} [\text{VG}(\phi_z, \phi_w, \beta, \theta)]$, and $\text{VG} (\phi_z, \phi_w, \beta, \theta) \geq 0$ (it is a KL divergence), shows optimality.

From \eqref{eq:trivial-solution} we have that $p_{\theta^*} (x|z) = p_\text{data} (x)$ for all $x$ and $z$ and thus with \eqref{eq:generative-equals-data} we have
\begin{align}
p_{\theta^*, \beta^*} (z,w,v|x,y=1) &= \frac{p_{\theta^*} (x|z, w, v, y=1) p_{\beta^*} (z, w, v| y=1)}{p_{\theta^*, \beta^*} (x|y=1)} \\
&= \frac{p_{\theta^*} (x|z) p_{\beta^*} (z, w, v| y=1)}{p_\text{data} (x)}  \\
&=p_{\beta^*} (z, w, v| y=1). \label{eq:independent-x}
\end{align}
The reconstruction term $p_\theta (x|z, w, v, y=1)  = p_\theta (x|z)$ for every $\theta$ because in GMVAE, data reconstruction depends only on $z$ and is independent of $w$ and $v$ (see \S 3.1 of the paper).

Also from Bayes' and GMVAE's generative model factorization, we have the following simplification
\begin{align}
p_{\beta^*} ( v|z, w, y=1)  &=  \frac{p_{\beta^*} ( z|w,y=1,v)  p(v |y=1) p(w)}{p_{\beta^*}(z,w|y=1)} \\
&=  \frac{p_{\beta^*} ( z|w,y=1,v)  p(v |y=1) p(w)}{p_{\beta^*}(z|w,y=1) p(w|y=1)} \\
&=  \frac{p_{\beta^*} ( z|w,y=1,v)  p(v | y=1)}{\sum_{v'} p_{\beta^*}(z|w,y=1,v') p(v'|y=1) } \label{eq:v-factor}   \\
&= p (v|y=1) \label{eq:v-post-prior}
\end{align}
where \eqref{eq:trivial-solution} is only used in the last line. Substituting \eqref{eq:independent-x} into $\text{VG}(\phi_z^*, \phi_w^*, \beta^*, \theta^*)$ we obtain
\begin{align}
& \text{VG}(\phi_z^*, \phi_w^*, \beta^*, \theta^*)\\
&=\E_{q_{\phi^*}(v,w,z | x,y=1)} \left[\log \frac{q_{\phi^*}(v,w,z | x,y=1)}{p_{\theta^*, \beta^*}  (z,w,v | x,y=1)}  \right]  \\
&=  \E_{q_{\phi^*}(v,w,z |x, y=1)} \left[\log \frac{q_{\phi^*}(v,w,z | x,y=1)}{p_{\beta^*} (z, w, v| y=1)}  \right]   \\
&=  \E_{p_{\beta^*} ( v|z, w, y=1) q_{\phi_w^*}(w|x,y=1) q_{\phi_z^*}(z|x)} \left[ \log \frac{p_{\beta^*} ( v|z, w, y=1) q_{\phi_w^*}(w|x,y=1) q_{\phi_z^*}(z|x)}{p_{\beta^*} ( z|w, y=1, v) p(w) p (v|y=1) }  \right]   \\
 &=  \mathbb{E}_{q_{\phi_w^*}(w|x,y=1) q_{\phi_z^*}(z|x)  } \left[ \log q_{\phi_z^*}(z|x)  - \sum_{j=1}^K p_{\beta^*} ( v=j|z, w, y=1)  \log p_{\beta^*} ( z|w, y=1, v=j)  \right] \\
&+ KL(q_{\phi_w^*}(w|x,y=1) || p(w) ) \\
&+ \mathbb{E}_{q_{\phi_w^*}(w|x,y=1) q_{\phi_z^*}(z|x)} \left[ KL( p_{\beta^*} ( v|z, w, y=1)  || p (v|y=1) ) \right]  \\
&= 0
\end{align}
due to \eqref{eq:trivial-solution} and \eqref{eq:v-post-prior}. To complete the proof, simply note that negating \eqref{eq:l-vg} yields \\
$-\E_\mathcal{X} [ \mathcal{L} (K;\phi_z^*, \phi_w^*, \beta^*, \theta^*)]=-\E_\mathcal{X} \left[\log p_\text{data} \right] $.
\end{proof}

\section{Proof of Proposition 2}
\begin{lemma} \label{thm:delta-exist}
For every $\delta >0$ and $\mu$, there exists $\sigma^2$ such that if $f (z)$ is the pdf of a $d$-dimensional Normal random vector with mean $\mu$ and diagonal covariance $\sigma^2$ then
\begin{align}
f (z) &\leq \delta \quad \text{for every $z$.}
\end{align}
\end{lemma}
\begin{proof}
Let $u = \left( \frac{1}{\delta} (2\pi)^{-d/2}  \right)^{1/d}$ and $ \sigma = \left( u, ..., u \right)$. We have \phantom{\qedhere}
\begin{equation}
f (z) =\prod_i \frac{1}{\sigma_i\sqrt{2 \pi}} \exp\left\{ -\frac{1}{2 \sigma_i^2} (z_i-\mu_i)^2  \right\} \leq \prod_i \frac{1}{\sigma_i\sqrt{2 \pi} } = \delta.   \tag*{\qed} 
\end{equation}
\end{proof}
\begin{proposition}
Let us assume $C=1$, Assumptions 1 and 2 hold, and that $p(v|y=1)$ is uniform in the appropriate dimension. We have 
\begin{align}
\min\left\{ -\E_\mathcal{X} [\mathcal{L}(K; \phi_z, \phi_w, \beta, \theta)] \right\} - \min \left\{ -\E_\mathcal{X}[\mathcal{L}(K+1; \phi_z, \phi_w, \beta, \theta)]\right\} \geq \epsilon_K
\end{align}
where $-\log 2 \leq \log(K/(K+1)) \leq \epsilon_K$ for all $K$.
\end{proposition}
\begin{proof}
We show that for every solution $(\phi_z', \phi_w', \beta', \theta')$ to $\min \E_\mathcal{X} [ - \mathcal{L} (K;\phi_z, \phi_w, \beta, \theta)]$, there exists a corresponding solution $(\phi_z^*, \phi_w^*, \beta^*, \theta^*)$ such that 
\begin{align}
-\E_\mathcal{X} [ \mathcal{L} (K;\phi_z', \phi_w', \beta', \theta')]  = -\E_\mathcal{X} [ \mathcal{L} (K+1;\phi_z^*, \phi_w^*, \beta^*, \theta^*) ] + \epsilon_K.
\end{align}
Let us assume that $(\phi_z', \phi_w', \beta', \theta')$ minimizes $ -\E_\mathcal{X} [ \mathcal{L} (K;\phi_z, \phi_w, \beta, \theta)] $. Then we can choose 
 \begin{align}
\left.
\begin{array}{ll}
\phi_z^* &= \phi_z' \\
\phi_w^* &= \phi_w' \\
\theta^* &= \theta'
\end{array}
\right. \label{eq:same}
\end{align}
which is a valid choice by Assumption 2, and have $\beta^*$ such that
 \begin{align}
p_{\beta^*} ( z|w, y=1, v) &=  p_{\beta'} ( z|w, y=1, v)  \quad \text{for all $v\leq K$} \label{eq:first-K-same} \\
p_{\beta^*} ( z|w, y=1, v=K+1) &\leq \delta \quad \text{ for every }  z,w \label{eq:K+1-far-away}
\end{align}
for any fixed $0<\delta < 1/e $. Conditions \eqref{eq:first-K-same} and \eqref{eq:K+1-far-away} are always possible due to Assumptions 1 and 2 and Lemmas \ref{thm:any-constant} and \ref{thm:delta-exist}. In essence, we choose $\beta^*$ such that the first $K$ subcluster generative distributions are the same as the case $\beta'$ but we take the $(K+1)$-th subcluster generative distribution to map all points $w$ to the same Normal distribution with large enough covariance. 

Inserting \eqref{eq:first-K-same} and \eqref{eq:K+1-far-away} into \eqref{eq:v-factor} and combined with uniform priors, we get that
\begin{equation}
p_{\beta^*} ( v=K+1|z, w, y=1) = \frac{p_{\beta^*} ( z|w, y=1, v=K+1) }{\sum_{j=1}^K p_{\beta'} ( z|w, y=1, v=j) +  p_{\beta^*} ( z|w, y=1, v=K+1)} \label{eq:v-K+1-rewrite}
\end{equation}
and
\begin{align}
p_{\beta^*} ( v=k|z, w, y=1)&= \frac{p_{\beta'} ( z|w, y=1, v=k)}{\sum_{j=1}^K p_{\beta'} ( z|w, y=1, v=j) +  p_{\beta^*} ( z|w, y=1, v=K+1) } \\
&\leq  \frac{p_{\beta'} ( z|w, y=1, v=k)}{\sum_{j=1}^K p_{\beta'} ( z|w, y=1, v=j)}  
= p_{\beta'} ( v=k|z, w, y=1) \label{eq:v-post-inq}
 \end{align}
 for all $k \leq K$. The absolute difference between the two posteriors for $k \leq K$ in \eqref{eq:v-post-inq} is bounded by a factor of $\delta$ as follows:
\begin{align}
&\bigg| p_{\beta^*} ( v=k|z, w, y=1) - p_{\beta'} ( v=k|z, w, y=1)\bigg|   \\
&=  \bigg| \frac{p_{\beta'} ( z|w, y=1, v=k)}{\sum_{j=1}^K p_{\beta'} ( z|w, y=1, v=j) +  p_{\beta^*} ( z|w, y=1, v=K+1) } \\
& \qquad \qquad   \qquad \qquad   \qquad \qquad   \qquad \qquad  -  \frac{p_{\beta'} ( z|w, y=1, v=k)}{\sum_{j=1}^K p_{\beta'} ( z|w, y=1, v=j)  }\bigg|  \\
&= \frac{p_{\beta^*} ( z|w, y=1, v=K+1) p_{\beta'} ( z|w, y=1, v=k) }{\left(\sum_{j=1}^K p_{\beta'} ( z|w, y=1, v=j) +  p_{\beta^*} ( z|w, y=1, v=K+1) \right) } \\
&\qquad \qquad \qquad \qquad\qquad \qquad\qquad \qquad  \times \frac{1}{\sum_{j=1}^K p_{\beta'} ( z|w, y=1, v=j)}  \\
& \leq \delta   \frac{p_{\beta'} ( z|w, y=1, v=k) }{\left(\sum_{j=1}^K p_{\beta'} ( z|w, y=1, v=j)  \right)^2}  \\
&= \delta  A(z, w, v=k) \label{eq:bound-diff}\>.
\end{align}

Now we calculate $\epsilon_K$ given by
\begin{align}
\E_\mathcal{X} [ - \mathcal{L} (K;\phi_z', \phi_w', \beta', \theta')] -  \E_\mathcal{X} [- \mathcal{L} (K+1;\phi_z^*, \phi_w^*, \beta^*, \theta^*) ]  = \epsilon_K
 \end{align}
 Because of \eqref{eq:same}, $\epsilon_K$ simplifies to
 \begin{align}
 & \epsilon_K = \\
&-\E_\mathcal{X} \left[  \mathbb{E}_{q_{\phi_w^*}(w|x,y=1) q_{\phi_z^*}(z|x)  } \left[    \sum_{j=1}^K p_{\beta'} ( v=j|z, w, y=1)  \log p_{\beta'} ( z|w, y=1, v=j)  \right] \right]  \\ 
 &+\E_\mathcal{X} \left[  \mathbb{E}_{q_{\phi_w^*}(w|x,y=1) q_{\phi_z^*}(z|x)  } \left[\sum_{j=1}^{K+1} p_{\beta^*} ( v=j|z, w, y=1)  \log p_{\beta^*} ( z|w, y=1, v=j)  \right] \right] \\
&+\E_\mathcal{X} [\mathbb{E}_{q_{\phi_w^*}(w|x,y=1) q_{\phi_z^*}(z|x)} \left[ KL( p_{\beta'} ( v|z, w, y=1)  || p_{K} (v|y=1) ) \right] ] \\
& -\E_\mathcal{X} [\mathbb{E}_{q_{\phi_w^*}(w|x,y=1) q_{\phi_z^*}(z|x)} \left[ KL( p_{\beta^*} ( v|z, w, y=1)  || p_{K+1} (v|y=1) ) \right] ] \\
& = \epsilon_K^{(1)} + \epsilon_K^{(2)}
 \end{align}
where $p_K (v|y=1)$ indicates that $v$ is $K$-dimensional, and $\epsilon_K^{(1)}$ are the first two terms while $\epsilon_K^{(2)}$ are the the last two terms.

We first analyze $\epsilon_K^{(1)}$. For brevity, we combine the expectations and simply write $\E[\cdot]$. Together with \eqref{eq:first-K-same}, \eqref{eq:v-K+1-rewrite}, and \eqref{eq:bound-diff}, we get
\begin{align}
\bigg|  \epsilon_K^{(1)}  \bigg| &= \bigg|  - \E \left[ \sum_{j=1}^K p_{\beta'} ( v=j|z, w, y=1)  \log p_{\beta'} ( z|w, y=1, v=j)  \right] \\ 
&+\E \left[ \sum_{j=1}^{K} p_{\beta^*} ( v=j|z, w, y=1)  \log p_{\beta'} ( z|w, y=1, v=j) \right] \\ 
&+\E \left[  p_{\beta^*} ( v=K+1|z, w, y=1)  \log p_{\beta^*} ( z|w, y=1, v=K+1) \right]  \bigg| \\
&=   \bigg|  \E \left[  \sum_{j=1}^K  \log p_{\beta'} ( z|w, y=1, v=j) \left( p_{\beta^*} ( v=j|z, w, y=1)  -p_{\beta'} ( v=j|z, w, y=1)  \right) \right]   \\
&+\E \left[  \frac{p_{\beta^*} ( z|w, y=1, v=K+1) \log p_{\beta^*} ( z|w, y=1, v=K+1)  }{\sum_{j=1}^K p_{\beta'} ( z|w, y=1, v=j) +  p_{\beta^*} ( z|w, y=1, v=K+1)} \right]  \bigg|   \\
&\leq  \delta \cdot  \E \left[  \sum_{j=1}^K  \bigg|  \log p_{\beta'} ( z|w, y=1, v=j)\bigg| A (z,w,v=j)  \right]    \\
&+  |\delta( \log \delta) | \E \left[  \frac{1}{\sum_{j=1}^K p_{\beta'} ( z|w, y=1, v=j)}  \right] =o(1)\>,\label{eq:delta-log-delta}
\end{align}
where the last inequality follows from $|x\log x|$ being increasing for $x\le 1/e$ and in $o(1)$ we consider $\delta \rightarrow 0$.

Next we study $\epsilon_K^{(2)}$. For shorthand, let us define
\begin{align}
&\log \left( (K+1) p_{\beta^*} ( v=K+1|z, w, y=1)\right) \\
&= \log \left(\frac{(K+1)  p_{\beta^*} ( z|w, y=1, v=K+1)}{\sum_{j=1}^K p_{\beta'} ( z|w, y=1, v=j) + p_{\beta^*} ( z|w, y=1, v=K+1)} \right) \\
&= \log p_{\beta^*} ( z|w, y=1, v=K+1)  + B(z,w)
\end{align}
and note that
\begin{align}
&|B(z,w)|  =\bigg| \log \left(\frac{(K+1)  }{\sum_{j=1}^K p_{\beta'} ( z|w, y=1, v=j) + p_{\beta^*} ( z|w, y=1, v=K+1)} \right)  \bigg| \\
& \leq \max\left\{  \bigg| \log \left(\frac{(K+1)  }{\sum_{j=1}^K p_{\beta'} ( z|w, y=1, v=j)} \right) \bigg|, \right. \\
&\qquad \qquad  \left. \bigg| \log \left(\frac{(K+1)  }{\sum_{j=1}^K p_{\beta'} ( z|w, y=1, v=j)+1/e} \right) \bigg|   \right\} \\
&= C(z,w).
\end{align}
We have
 \begin{align}
&\epsilon_K^{(2)}\\
&= \E \left[ \sum_{j=1}^K p_{\beta'} ( v=j|z, w, y=1) \log \left( K p_{\beta'} ( v=j|z, w, y=1) \right)\right] \\
&-\E \left[ \sum_{j=1}^K p_{\beta^*} ( v=j|z, w, y=1) \log \left( (K+1) (p_{\beta^*} ( v=j|z, w, y=1) )\right)\right]  \\
 &-\E \left[ p_{\beta^*} ( v=K+1|z, w, y=1)\log \left( (K+1) p_{\beta^*} ( v=K+1|z, w, y=1)\right)\right] \\
 &=\E \left[ \sum_{j=1}^K (\log  K) p_{\beta'} ( v=j|z, w, y=1) -  (\log  (K+1)) p_{\beta^*} ( v=j|z, w, y=1) \right] \\
&+\E  \Bigg[ \sum_{j=1}^K p_{\beta'} ( v=j|z, w, y=1) \log  p_{\beta'} ( v=j|z, w, y=1) \\
&\qquad \qquad \qquad \qquad \qquad \quad-  p_{\beta^*} ( v=j|z, w, y=1)  \log  p_{\beta^*} ( v=j|z, w, y=1) \Bigg] \\
&-\E \left[ p_{\beta^*} ( v=K+1|z, w, y=1)\log \left( (K+1) p_{\beta^*} ( v=K+1|z, w, y=1)\right)\right] \\
& \geq \log (K) - (\log  (K+1))  \E\left[   \sum_{j=1}^K p_{\beta^*} ( v=j|z, w, y=1) \right] \\
&+\E \left[ \sum_{j=1}^K (p_{\beta'} ( v=j|z, w, y=1)-  p_{\beta^*} ( v=j|z, w, y=1))   \log  (p_{\beta'} ( v=j|z, w, y=1) ) \right] \label{eq:explain1}  \\
&- \bigg| \E\left[  \frac{p_{\beta^*} ( z|w, y=1, v=K+1) \log p_{\beta^*} ( z|w, y=1, v=K+1) }{\sum_{j=1}^K p_{\beta'} ( z|w, y=1, v=j) +  p_{\beta^*} ( z|w, y=1, v=K+1)}  \right]  \bigg| \\
&- \bigg| \E\left[  \left( \frac{p_{\beta^*} ( z|w, y=1, v=K+1)}{\sum_{j=1}^K p_{\beta'} ( z|w, y=1, v=j) +  p_{\beta^*} ( z|w, y=1, v=K+1)}  \right)B(z,w)\right]  \bigg| \\
& \geq \log (K) -  \log  (K+1) \\
&- \delta \cdot \E \left[\sum_{j=1}^K A(z,w,v=j) \bigg| \log  (p_{\beta'} ( v=j|z, w, y=1) ) \bigg| \right]  \label{eq:explain2}\\
&- \delta  \left(\log \delta \right)  \E\left[    \frac{1}{\sum_{j=1}^K p_{\beta'} ( z|w, y=1, v=j) } \right] \label{eq:explain3}  \\
&-\delta \cdot \E\left[     \frac{1}{\sum_{j=1}^K p_{\beta'} ( z|w, y=1, v=j) }  C(z,w)\right] \\
& =\log \frac{K}{K+1} + o(1)\>.
\end{align}
In \eqref{eq:explain1} we use \eqref{eq:v-post-inq}, in \eqref{eq:explain2} we rely on \eqref{eq:bound-diff}, and in \eqref{eq:explain3} we use \eqref{eq:delta-log-delta} again.

To summarize, we have $\epsilon_K \ge -|\epsilon_K^{(1)}| + \epsilon_K^{(2)} \ge -o(1) + o(1) + \log \frac{K}{K+1} = \log \frac{K}{K+1} + o(1)$. Thus $\epsilon_K \ge \log \frac{K}{K+1}$.
\end{proof}